\PassOptionsToPackage{table}{xcolor}

\documentclass[letterpaper, 10 pt, conference]{class/ieeeconf}  

\IEEEoverridecommandlockouts                              
\usepackage{anyfontsize}
\usepackage{ stmaryrd }
\usepackage{amsmath,amsfonts,bm}
\usepackage{cuted}
\usepackage[ruled,vlined,linesnumbered]{algorithm2e}
\usepackage{graphicx}
\usepackage{tabularx} 

\usepackage{amssymb}
\usepackage{amsthm}
\usepackage{multirow}

\usepackage{floatrow}
\usepackage{color}
\usepackage[bb=boondox]{mathalfa}
\usepackage{dsfont}
\usepackage{upgreek}
\usepackage{mathrsfs}
\usepackage{url}
\usepackage{accents}

\usepackage{pifont}
\def\eqref#1{equation~\ref{#1}}

\def\1{\bm{1}}

\DeclareMathAlphabet{\mathsfit}{\encodingdefault}{\sfdefault}{m}{sl}
\SetMathAlphabet{\mathsfit}{bold}{\encodingdefault}{\sfdefault}{bx}{n}

\usepackage[table,xcdraw]{xcolor}
\usepackage{cite}
\usepackage{latexsym}

\usepackage{flushend}
\usepackage{lipsum}

\newtheorem{proposition}{Proposition}

\newtheorem{remark}{Remark}

\newcommand{\etal}{\textit{et al.}}
\newcommand{\improvecolor}{\color[HTML]{3B9612}}

\usepackage{amsthm}

\makeatletter
\let\NAT@parse\undefined
\makeatother
\usepackage[bookmarks=false, linkcolor=blue, urlcolor=blue, citecolor=blue]{hyperref} 

\hypersetup{
    colorlinks=true,
    linkcolor=red,
    filecolor=magenta,      
    urlcolor=blue,
    pdfstartview={FitH},
    citecolor =blue
    }

\title{\LARGE \bf Lightweight Temporal Transformer Decomposition for\\ Federated Autonomous Driving}

\author{Tuong Do$^{1,2,3}$, Binh X. Nguyen$^{1}$, Quang D. Tran$^{1,3}$, Erman Tjiputra$^{1}$, Te-Chuan Chiu$^{2}$, Anh Nguyen$^{3}$
\thanks{$^{1}$ AIOZ, Singapore 
        {\tt\small tuong.khanh-long.do@aioz.io}}%
\thanks{$^{2}$ Department of Computer Science, NTHU, Taiwan}%
\thanks{$^{3}$ Department of Computer Science, University of Liverpool, UK}
}

\begin{document}

\maketitle
\thispagestyle{empty}
\pagestyle{empty}

\begin{abstract}
Traditional vision-based autonomous driving systems often face difficulties in navigating complex environments when relying solely on single-image inputs. To overcome this limitation, incorporating temporal data such as past image frames or steering sequences, has proven effective in enhancing robustness and adaptability in challenging scenarios. While previous high-performance methods exist, they often rely on resource-intensive fusion networks, making them impractical for training and unsuitable for federated learning. To address these challenges, we propose lightweight temporal transformer decomposition, a method that processes sequential image frames and temporal steering data by breaking down large attention maps into smaller matrices. This approach reduces model complexity, enabling efficient weight updates for convergence and real-time predictions while leveraging temporal information to enhance autonomous driving performance. Intensive experiments on three datasets demonstrate that our method outperforms recent approaches by a clear margin while achieving real-time performance. Additionally, real robot experiments further confirm the effectiveness of our method. \end{abstract}

\section{Introduction}
Autonomous driving has the potential to revolutionize transportation by significantly improving safety, efficiency, and convenience~\cite{gidado2020survey,nguyen2021autonomous} for human drivers. Central to the effectiveness of autonomous vehicles is their ability to process and interpret visual data to make accurate driving decisions. However, traditional vision-based autonomous driving systems face privacy concerns, as they require collecting data from multiple users to train the model~\cite{nguyen2021federated}. Furthermore, while recent studies have introduced various methods for autonomous driving, many of them predict trajectory information from a single image input~\cite{nguyen2021autonomous}. This limitation reduces the system’s ability to respond quickly and safely while maintaining the privacy of the users' data~\cite{zhao2021end}.

To overcome the limitation when using a single frame as input for the network, several works have included a sequence of frames to predict directly the steering control signal~\cite{abou2019multimodalHPO, hu2022st}. This approach enables the system to anticipate potential hazards and take preventive measures, such as adjusting speed or changing lanes, to avoid close encounters. Despite the potential benefits, integrating temporal information into autonomous driving systems presents several challenges. 
In particular, the recent model complexity may necessitate substantial data for training, impede integration on devices with limited computational power, and pose significant challenges in ensuring real-time responses~\cite{wang2022and}.


\begin{figure}[t]
   \centering
\setlength{\tabcolsep}{0pt}
\begin{tabular}{c}

\shortstack{\includegraphics[width=0.95\linewidth]{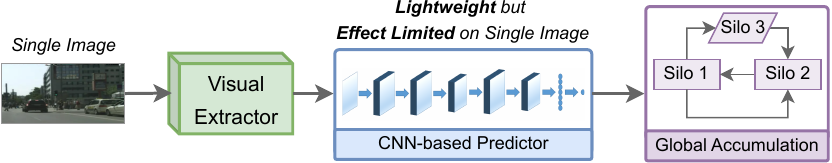}}\\
\shortstack{\small (a)}\\
\shortstack{\includegraphics[width=0.95\linewidth]{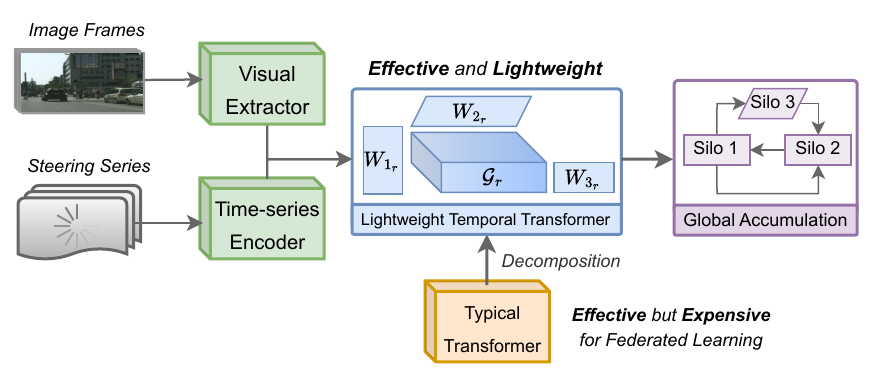}}\\
\shortstack{\small (b)}\\
\end{tabular}
\vspace{-1ex}
    \caption{Comparison between traditional single-frame solutions for federated autonomous driving (a) and our lightweight temporal transformer network to enhance the training feasibility in federated learning setup (b).
    \vspace{-5ex}
    }
    \label{fig:intro}

\end{figure}

To address data privacy, several autonomous driving approaches utilize federated learning to train decentralized models across multiple vehicles~\cite{barbieri2021decentralized,zeng2021multi,meese2022bfrt}. However, most autonomous driving models still rely on the single frame as input and develop a relatively simple network to enable feasible a training in federated learning setup~\cite{li2021privacy,nguyen2022deep,liang2022federated}. This single-frame approach overlooks the temporal data that each vehicle collects over time, which can provide essential context for understanding motion patterns, tracking objects, and anticipating potential hazards. As a result, these models do not fully leverage the sequence of information needed to better predict and respond to dynamic driving scenarios,
ultimately limiting their performance and adaptability.

In this paper, our goal is to develop a federated autonomous driving framework that takes into account the temporal information as the input. To address the complexity and learning challenges of the fusion model when training in a federated scenario using temporal information, we propose a Lightweight Temporal Transformer, a new approach designed to reduce the complexity of the network in each silo by efficiently approximating the information from the inputs. Our method utilizes a decomposition method under unitary attention to break down learnable attention maps into low-rank ones, ensuring that the resulting models remain lightweight and trainable. By reducing model complexity, our approach enables the network to use temporal data while ensuring convergence. Intensive experiments demonstrate that our approach clearly improves performance over state-of-the-art methods in federated autonomous driving.

\section{Related works}
\vspace{-1ex}
\textbf{Autonomous Driving.} Autonomous driving is a rapidly advancing field that has garnered significant research interest in recent years. Various studies have explored the application of deep learning for critical tasks, including object detection and tracking~\cite{hu2022investigating,chen2022pseudo}, trajectory prediction~\cite{wang2022ltp}, and autonomous braking and steering control~\cite{ijaz2021automatic}. For example, Xin \etal~\cite{xin2020slip} proposed a recursive backstepping steering controller that connects yaw-rate-based path-following commands to steering adjustments, while Xiong \etal~\cite{xiong2021reduced} analyzed nonlinear dynamics using proportional control methods. Yi \etal~\cite{yi2022anti} introduced an algorithm that determines the instantaneous center of rotation within a self-reconfigurable robot's area, enabling waypoint navigation while avoiding collisions. Additionally, Yin \etal~\cite{yin2022trajectory} combined model predictive control with covariance steering theory to create a robust controller for nonlinear autonomous driving systems. 
Moreover, recent works have leveraged temporal information to address complex environments and dynamic scene changes, demonstrating improved robustness and adaptability in challenging scenarios~\cite{shao2023reasonnet}.
However, despite these advances~\cite{liu2024attention}, managing model complexity to enable deployment on low-level devices while maintaining effective performance remains a significant challenge.

\textbf{ Federated Learning.}
Federated learning (FL) supports decentralized training of machine learning models across multiple devices while keeping data localized, thereby preserving privacy and reducing data transfer~\cite{kaur2024federated}. In autonomous driving, FL enables vehicles to collaboratively learn from diverse datasets without sharing raw data~\cite{nguyen2022deep}. Previous research has explored the use of FL in autonomous driving~\cite{do2024reducingCDL,zeng2021multi,meese2022bfrt}. Recently, Zhao \etal~\cite{zhao2018federated} developed a federated learning framework for vehicle-to-vehicle communication that enhances model robustness and generalization. Some recent works also consider temporal information to improve performance in complex environments~\cite{zhang2022gof, zhou2022stfl,shen2024spatial,liu2024online}. Other works have explored clustering-based solutions for post-processing~\cite{belal2022pepper, kaur2024federated}, learning feasibility through the modifying of the accumulator~\cite{yuan2022fedtse}, topology design~\cite{chen2020practicalPriRec}, or global architecture~\cite{meese2022bfrt, zeng2021multi}. However, fully exploiting temporal information within the constraints of federated learning remains a significant challenge due to computational complexity and limited device resources~\cite{chellapandi2023federated}.  


\textbf{Lightweight Models.}
The tensor decomposition techniques aim at breaking down complex interactions into simpler components~\cite{kolda2009tensor}, thus, reducing the computational burden and improving the interpretability of models~\cite{sidiropoulos2017tensor}. This technique shows promise in various applications, including image recognition~\cite{zhang2020robust,yin2021towards,dai2023deep} and natural language processing~\cite{tjandra2020recurrent,wang2021kronecker}, but its potential in federated learning for autonomous driving remains largely untapped. Compared with distillation~\cite{hinton2002training,sautier2022image,li2022driver}, pruning~\cite{im2023visual,yang2023deep,samal2020attention}, or quantization~\cite{gheorghe2021model,sciangula2022hardware} that require complex training setups, decomposing the network tensor can be trained directly with less parameters without the need to modify training paradigm, which shows potentials in federated training for autonomous vehicles when handling high dimensional data inputs.

\section{Methodology}

\subsection{Preliminary}
We consider a federated network with $N$ autonomous vehicles, collaboratively training a global driving policy $\theta$ by aggregating local weights $\theta_i$ from each silo $i$, where $i \in [1, N]$. Each silo minimizes a regression loss $\mathcal{L}$ computed using a deep network that predicts the steering angle from temporally ordered RGB images and steering series.

\begin{figure*}[ht]
    \centering
    \includegraphics[width = \columnwidth, keepaspectratio=True]{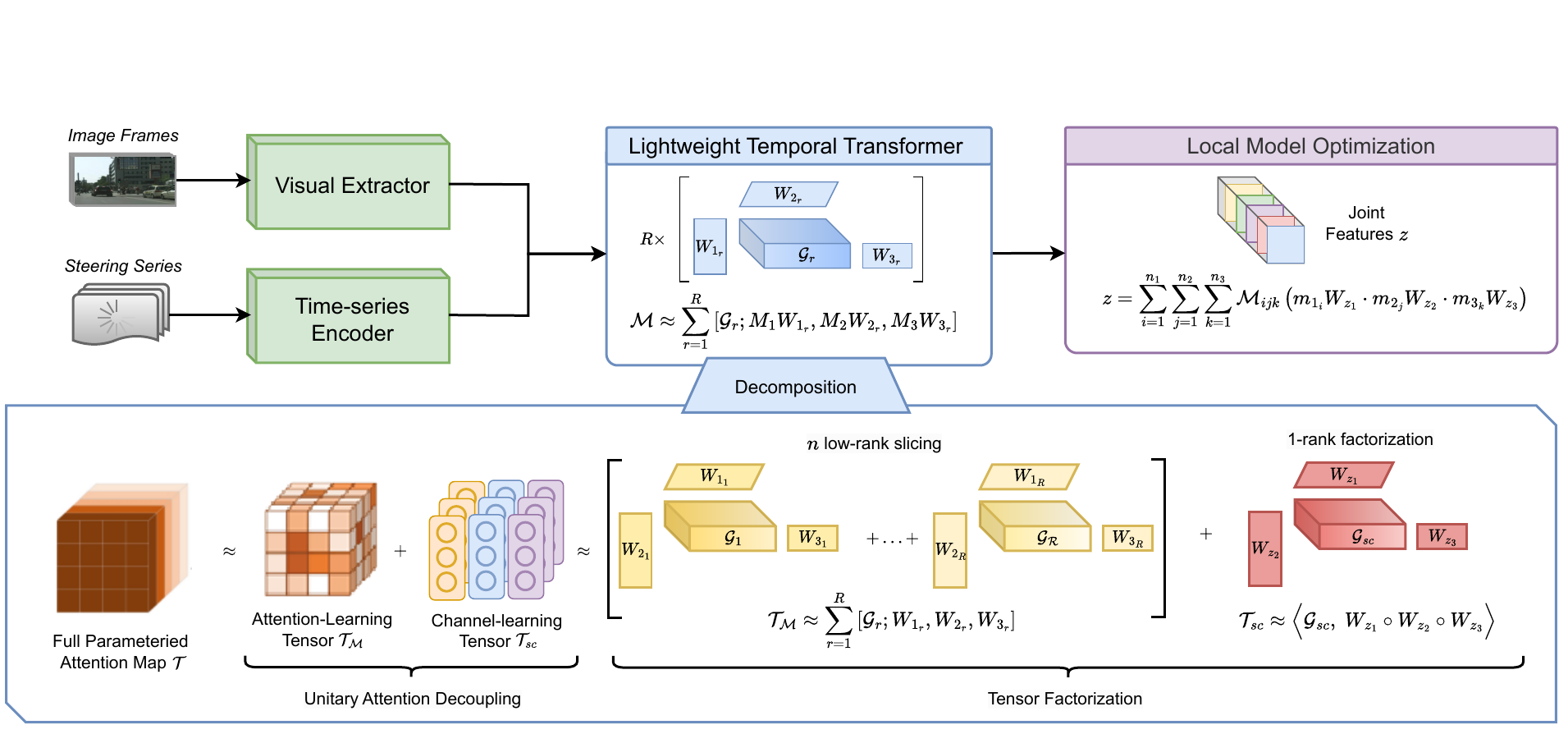}
    \caption{An overview of our lightweight temporal transformer decomposition method for federated autonomous driving.
    \vspace{-3ex}
    }
    \label{fig:tensor_based_PARALIND}
\end{figure*}

\textbf{Local Regression Objective.}
We use mean squared error (MSE)  as the objective function for predicting the steering angle in each local silo. Here, we only use the extracted joint features $z$ of the local model to predict steering angles.
\begin{equation}
\mathcal{L} = \text{MSE}(\theta_i, {\xi}^b_i  )
\label{eq:MSE}
\end{equation}
where $b$ is the mini-batch size; ${\xi}^b_i$  is the ground-truth steering angle of batch $b$ from silo $i$.

\textbf{Local Optimization.}
To ensure the model convergence under a federated training scenario, for each $k$ communication round, we use decentralized periodic averaging stochastic gradient descent (DPASGD)~\cite{wang2018cooperative}. 
\begin{equation}
\theta_{i}\left(k + 1\right) =
\begin{cases}
    \sum_{j \in \mathcal{N}_{i}^{+} \cup{\{i\}}}\textbf{A}_{i,j}\theta_{j}\left(k\right), \\\qquad \text{if k} \equiv 0 \left(\text{mod }u + 1\right) \text{\&} \left|\mathcal{N}_{i}^{+}\right| > 1 ,\\
    \theta_{i}\left(k\right)-\alpha_{k}\frac{1}{b}\sum^b_{h=1}\nabla \mathcal{L}_i\left(\theta_{i}\left(k\right),\xi^b_i\left(k\right)\right), \\\qquad\text{otherwise.}
\end{cases}
\label{eq:skipbatch_DFL}
\end{equation}
where $\mathcal{N}_i^{+}$ is the in-neighbors set of silo $i$; $u$ is the number of local updates; $\textbf{A}$ is a consensus matrix for parameter accumulating. 

\textbf{Global Accumulation.}
Since our method focuses on the practical application of a network under a federated learning scenario, we use the simple accumulation solution FedAvg~\cite{li2019convergence} for computing the global model $\theta$. The federated averaging process is conducted as follows:
\begin{equation}
\theta = \frac{1}{\sum^N_{i=0}{\lambda_i}} \sum^N_{i=0}\lambda_{{i}} \theta_{{i}},
\label{eq:aggr_model}
\end{equation}
where 
$\lambda_i = \{0,1\}$. Note that $\lambda_i = 1$ indicates that silo $i$ joins the inference process and $\lambda_i = 0$ if not.

\textbf{Feature Extraction.} We use a standard vision transformer to extract the feature from the sequence of temporal inputs. This representation, $z \in \mathds{R}^{d_z}$, is computed as:
\begin{equation}
z = \Big\langle \mathcal{T},\; \operatorname{vec}(M_1) \circ \operatorname{vec}(M_2) \circ \operatorname{vec}(M_3) \Big\rangle,
\label{eq:hypothesis}
\end{equation}
where $\circ$ is the outer product; $\langle.,.\rangle$ is the inner product;
$\mathcal{T} \in \mathds{R}^{d_{M_1} \times d_{M_2} \times d_{M_3} \times d_z}$ is a learnable tensor; $M_l \in \mathds{R}^{n_l \times d_l}$ is the modality input with $n_l$ elements, each represented by $d_l$-dimension features; $d_{M_l} = n_l \times d_l$; and $vec(M_l)$ vectorizes $M_l$ into a row vector. $M_1, M_2, M_3$ represent past frames, steering series, and the current RGB image, respectively.
While $\mathcal{T}$ captures input interactions, learning such a large tensor is impractical with high-dimensional inputs $d_{M_l}$, straining vehicle computing resources and hindering model convergence due to large linear parameter correlations. Thus, we aim to reduce the size of $\mathcal{T}$ by minimizing unnecessary linear combinations through the introduced Lightweight Temporal Transformer Decomposition. Specifically, we use \textit{unitary attention decoupling} to approximate large tensor $\mathcal{T}$ into smaller ones, followed by a \textit{tensor factorization} to factorize tensors into factor matrices.

\subsection{Unitary Attention Decoupling}
Inspired by \cite{Yang2016StackedAN}, we rely on the idea of \textit{unitary attention} mechanism to reduce the size of  $\mathcal{T}$. Specifically, let $z_p \in \mathds{R}^{d_z}$ be the joint representation of $p^{th}$ triplet of channels where each channel in the triplet is from a different input. The representation of each channel in a triplet is $m_{1_i}, m_{2_j}, m_{3_k}$, where $i \in  [1,n_1], j \in  [1,n_2], k \in  [1,n_3]$, respectively. There are $n_1 \times n_2 \times n_3$ possible triplets over the three inputs. The joint representation $z_p$ resulted from a fully parameterized trilinear interaction over three channel representations $m_{1_i}, m_{2_j}, m_{3_k}$ of $p^{th}$ triplet is computed as
\begin{equation}
z_p = \Big\langle \mathcal{T}_{sc},\, m_{1_i} \circ m_{2_j} \circ m_{3_k} \Big\rangle,
\label{eq:triplet_compute}
\end{equation}
where $\mathcal{T}_{sc} \in \mathds{R}^{d_{1} \times d_{2} \times d_{3} \times d_z}$  is the learning tensor between channels in the triplet.

Following~\cite{Yang2016StackedAN}, the joint representation
$z$ is approximated by using joint representations of all triplets described in (\ref{eq:triplet_compute}) instead of using fully parameterized interaction over three inputs. Hence, we compute
\begin{equation}
z = \sum_p \mathcal{M}_p z_p,
\label{eq:Unitary}
\end{equation}
Note that in (\ref{eq:Unitary}), we compute a weighted sum over all possible triplets. The $p^{th}$ triplet is associated with a scalar weight $\mathcal{M}_p$. The set of $\mathcal{M}_p$ is called as the attention map $\mathcal{M}$, where $\mathcal{M} \in \mathds{R}^{n_1 \times n_2 \times n_3}$.

The attention map $\mathcal{M}$ 
resulted from a reduced parameterized trilinear interaction over three inputs $M_1, M_2$ and $M_3$ is computed as follows
\begin{equation}
\mathcal{M} = \Big\langle \mathcal{T}_\mathcal{M},\; M_1 \circ M_2 \circ M_3 \Big\rangle,
\label{eq:tri_attmap}
\end{equation}
where $\mathcal{T}_\mathcal{M} \in \mathds{R}^{d_1 \times d_2 \times d_3}$ is the learning tensor of attention map $\mathcal{M}$. 
Note that the learning tensor $\mathcal{T}_\mathcal{M}$ in (\ref{eq:tri_attmap}) has a reduced size compared to the learning tensor $\mathcal{T}$.

By integrating (\ref{eq:triplet_compute}) into  (\ref{eq:Unitary}), the joint representation $z$ in (\ref{eq:Unitary}) can be rewritten as
\begin{equation}
z= \sum_{i=1}^{n_1}\sum_{j=1}^{n_2}\sum_{k=1}^{n_3} \mathcal{M}_{ijk} \langle \mathcal{T}_{sc},\; m_{1_i} \circ m_{2_j} \circ m_{3_k} \rangle ,
\label{eq:transformed_hypo}
\end{equation}
where $\mathcal{M}_{ijk}$ in (\ref{eq:transformed_hypo}) is actually a scalar attention weight $\mathcal{M}_p$ of the attention map $\mathcal{M}$ in (\ref{eq:tri_attmap}).

It is also worth noting from (\ref{eq:transformed_hypo}) that to compute $z$, instead of learning the large tensor $\mathcal{T} \in \mathds{R}^{d_{M_1} \times d_{M_2} \times d_{M_3} \times d_z}$, we now only need to learn two smaller tensors  $\mathcal{T}_{sc} \in \mathds{R}^{d_{1} \times d_{2} \times d_{3} \times d_z}$  in (\ref{eq:triplet_compute}) and $\mathcal{T}_\mathcal{M} \in \mathds{R}^{d_1 \times d_2 \times d_3}$ in (\ref{eq:tri_attmap}).

\subsection{Tensor Factorization}
\label{sub:factorization}
Although the large tensor $\mathcal{T}$ is replaced by two smaller tensors $\mathcal{T}_\mathcal{M}$ and $\mathcal{T}_{sc}$, there are too many linear fusions between mentioned tensors which still affect the learning of the global model. Therefore, we apply the factorization as in~\cite{bro2009modelingPARALIND} to $\mathcal{T}_\mathcal{M}$ and $\mathcal{T}_{sc}$ into learnable factor matrices.  

The factorization for the learning tensor $\mathcal{T}_\mathcal{M} \in \mathds{R}^{d_1 \times d_2 \times d_3}$ can be calculated as
\begin{equation}
\mathcal{T}_\mathcal{M} \approx \sum_{r=1}^{\mathcal{R}} \Big\langle \mathcal{G}_r,\; W_{1_r} \circ W_{2_r} \circ W_{3_r} \Big\rangle,
\label{eq:tensorbased_PARALIND}
\end{equation}
where $\mathcal{G}_r \in \mathds{R}^{d_{1_r} \times d_{2_r} \times d_{3_r}}$ are compact learnable Tucker tensors~\cite{Hitchcock1927TheEOTucker}, small-sized tensors that support minimizing error when approximating a larger tensor using its factorized matrices; $\mathcal{R}$ is a slicing parameter, establishing a trade-off between the decomposition rate (which is directly related to the usage memory and the computational cost) and the performance. The maximum value for $\mathcal{R}$ is usually set to the greatest common divisor of $d_1, d_2$ and $d_3$. In our experiments, we found that $\mathcal{R}=32$ gives a good trade-off between the decomposition rate and the performance.

Here, we have dimension $d_{1_r} = d_1/\mathcal{R}$, $d_{2_r} = d_2/\mathcal{R}$  and $d_{3_r} = d_3/\mathcal{R}$; 
 $W_{1_r} \in \mathds{R}^{d_{1} \times d_{1_r}}$, $W_{2_r} \in \mathds{R}^{d_{2} \times d_{2_r}}$ and  $W_{3_r} \in \mathds{R}^{d_{3} \times d_{3_r}}$ are learnable factor matrices.  
Fig.~\ref{fig:tensor_based_PARALIND} shows the illustration of factorization for a tensor $\mathcal{T}_\mathcal{M}$.

The shortened form of $\mathcal{T}_\mathcal{M}$ in (\ref{eq:tensorbased_PARALIND}) can be rewritten as
\begin{equation}
\mathcal{T}_\mathcal{M}\approx \sum^\mathcal{R}_{r=1} {\llbracket  \mathcal{G}_r; W_{1_r},W_{2_r}, W_{3_r} \rrbracket},
\label{eq:shorten_PARALIND}
\end{equation}

Integrating the learning tensor $\mathcal{T}_\mathcal{M}$ from (\ref{eq:shorten_PARALIND}) into (\ref{eq:tri_attmap}), the attention map $\mathcal{M}$ can be rewritten as
\begin{equation}
\mathcal{M} = \sum^\mathcal{R}_{r=1} {\llbracket  \mathcal{G}_r; M_1W_{1_r},  M_2W_{2_r}, M_3W_{3_r}\rrbracket},
\label{eq:tri_attmap_decomp}
\end{equation}

Similar to $\mathcal{T}_\mathcal{M}$, we apply to $\mathcal{T}_{sc}$ in (\ref{eq:transformed_hypo}) to reduce the complexity. 
Note that the size of  $\mathcal{T}_{sc}$
 directly affects the dimension of the joint representation $z \in \mathds{R}^{d_z}$.
Hence, to minimize the loss of information, we set the slicing parameter $\mathcal{R} = 1$ and the projection dimension of factor matrices at $d_z$, i.e., the same dimension of the joint representation $z$.

Therefore,  $\mathcal{T}_{sc} \in \mathds{R}^{d_1 \times d_2 \times d_3 \times d_z}$ in (\ref{eq:transformed_hypo}) is  calculated as
\begin{equation}
\mathcal{T}_{sc} \approx  \Big\langle \mathcal{G}_{sc},\; W_{z_1} \circ W_{z_2} \circ W_{z_3} \Big\rangle,
\label{eq:T_sc_PARALIND}
\end{equation}
 where $W_{z_1} \in \mathds{R}^{d_{1} \times d_z}$, $W_{z_2} \in \mathds{R}^{d_{2} \times d_z}$, $W_{z_3} \in \mathds{R}^{d_{3} \times d_z}$ are learnable factor matrices and $\mathcal{G}_{sc} \in \mathds{R}^{d_z \times d_z \times d_z \times d_z}$ is a smaller  tensor (compared to $\mathcal{T}_{sc}$).

 Up to now, we already have $\mathcal{M}$ by (\ref{eq:tri_attmap_decomp}) and $\mathcal{T}_{sc}$ by (\ref{eq:T_sc_PARALIND}). Hence, we can compute $z$ using (\ref{eq:transformed_hypo}) as 
 \begin{equation}
\begin{aligned}
 &z= \sum_{i=1}^{n_1}\sum_{j=1}^{n_2}\sum_{k=1}^{n_3}\\
 & \mathcal{M}_{ijk}\langle \mathcal{G}_{sc},\; (m_{1_i}W_{z_1}) \circ (m_{2_j}W_{z_2}) \circ (m_{3_k}W_{z_3}) \rangle,
\end{aligned}
\label{eq:z1}
\end{equation}
Here, it is interesting to note that $\mathcal{G}_{sc} \in \mathds{R}^{d_z \times d_z \times d_z \times d_z}$ in (\ref{eq:z1}) has rank $1$. Thus, the result obtained from (\ref{eq:z1})  can be approximated by the Hadamard products without the presence of rank-1 tensor $\mathcal{G}_{sc}$ \cite{kolda2009tensor}.
In particular, $z$ in (\ref{eq:z1}) can be computed without using $\mathcal{G}_{sc}$ as
 \begin{equation}
z = \sum_{i=1}^{n_1}\sum_{j=1}^{n_2}\sum_{k=1}^{n_3} \mathcal{M}_{ijk} \left( m_{1_{i}}W_{z_1} \cdot m_{2_{j}}W_{z_2} \cdot m_{3_{k}}W_{z_3}\right),
\label{eq:final_hypo}
\end{equation}

The joint embedding dimension $d_z$ is a user-defined parameter that makes a trade-off between the capability of the representation and the computational cost. In our experiments, $d_z = 1,024$ gives a good trade-off.

\subsection{Convergence Analysis}
\begin{proposition} Lightweight Temporal Transformer Decomposition in Equation~\ref{eq:hypothesis} can be considered a form of Bilinear Attention~\cite{Kim2018BilinearAN}, naturally inheriting its convergence ability.
\label{prop1}
\end{proposition}

\begin{proof}
Let the input contain two representations of two modalities, i.e., $M^b_1 \in \mathds{R}^{n^b_1 \times d^b_1}$ and $M^b_2$ $ \in \mathds{R}^{n^b_2 \times d^b_2}$, where $n^b_1$ and $n^b_2$ are number of channels; $d^b_1$ and $d^b_2$ are the representation dimension of each corresponding channel. Following Equation~\ref{eq:hypothesis}, the joint representation $z \in \mathds{R}^{d_z}$ can now be described as
\begin{equation}
z= \Big\langle \mathcal{T}_b,\; \operatorname{vec}(M^b_1) \circ \operatorname{vec}(M^b_2) \Big\rangle,
\label{eq:ban_hypothesis}
\end{equation}
where $\mathcal{T}_{b}  \in \mathds{R}^{d^b_{n1} \times d^b_{n2} \times d_z}$ is learnable tensor;
$d^b_{n1} = n^b_1 \times d^b_1$; $d^b_{n2} =  n^b_2\times d^b_2$. By applying parameter factorization (Sec.~\ref{sub:factorization}), $z$ in (\ref{eq:ban_hypothesis}) can be approximated based on (\ref{eq:final_hypo}) as
\begin{equation}
z= \sum_{i=1}^{n^b_1}\sum_{j=1}^{n^b_2} \mathcal{M}_{ij}\left( {M^b_{1_i}}^{T}W_{z_{b1}} \cdot {M^b_{2_j}}^{T}W_{z_{b2}}\right),
\label{eq:decompose_BAN}
\end{equation}
where $W_{z_{b1}} \in \mathds{R}^{d^b_1 \times d_z}$ and $W_{z_{b2}} \in \mathds{R}^{d^b_2 \times d_z}$ are learnable factor matrices; $\mathcal{M}_{ij}$ is an attention weight of attention map $\mathcal{M} \in \mathds{R}^{n^b_1 \times n^b_2}$ which can be computed from  (\ref{eq:tri_attmap_decomp}) as
\begin{equation}
\mathcal{M} = \sum^\mathcal{R}_{r=1} {\llbracket  \mathcal{G}_r; {M^b_{1}}^{T}W_{1_r},  {M^b_{2}}^{T}W_{2_r}\rrbracket},
\label{eq:tri_att_BAN}
\end{equation}
where $W_{1_r} \in \mathds{R}^{d^b_1 \times d^b_{1_r}}$ and $W_{2_r} \in \mathds{R}^{d^b_2 \times d^b_{2_r}}$ are learnable factor matrices; $d^b_{1_r} = d^b_1 /\mathcal{R}$; $d^b_{2_r} = d^b_2 /\mathcal{R}$; each $\mathcal{G}_r \in \mathds{R}^{d^b_{1_r} \times d^b_{2_r}}$ is a learnable Tucker tensor. By extracting $k$-element and reorganize the multiplication computations over tensors, (\ref{eq:decompose_BAN}) can be rewritten as
\begin{equation}
z_k= \sum_{i=1}^{n^b_1}\sum_{j=1}^{n^b_2} \mathcal{M}_{ij}\left( {M^b_{1_{i}}}^{T}\left(W_{{z_1}_k} W_{{z_2}_k}^{T}\right)M^b_{2_{j}}\right),
\label{eq:final_hypo_VQA_BAN}
\end{equation}
where $z_k$ is $k^{th}$ element of the joint representation $z$; $W_{{z_1}_k}$ and $W_{{z_2}_k}$ are $k^{th}$ column in factor matrices $W_{z_1}$ and $W_{z_2}$.

Interestingly, from (\ref{eq:final_hypo_VQA_BAN}), we can rewrite it as a computational form of a Bilinear Attention as below:
\begin{equation}
\begin{aligned}
z_k& = \sum_{i=1}^{n^b_1}\sum_{j=1}^{n^b_2} \mathcal{M}_{ij}\left( {M^b_{1_{i}}}^{T}\left(W_{{z_1}_k} W_{{z_2}_k}^{T}\right)M^b_{2_{j}}\right)\\ &=  \sum_{i=1}^{n^b_1}\sum_{j=1}^{n^b_2} \mathcal{M}_{ij}\left( {M^b_{1_{i}}}^{T}W_{{z_1}_k} \right)\left(W_{{z_2}_k}^{T}M^b_{2_{j}}\right) \\
&= {({M^b_{1}}^{T} W_{z_1})}^T_k\mathcal{M} {({M^b_{2}}^{T}W_{z_2})_k}.
\label{eq:convertFunction}
\end{aligned}
\end{equation}
\vspace{0ex}
\end{proof}

\vspace{-4ex}
\begin{remark}

\begin{itemize}
	\item[] Proposition~\ref{prop1} suggests that the results of our decomposition method can be considered as a Bilinear Attention, which inherits its convergence ability and ensures the network will converge during the training.
\end{itemize}
\end{remark}

\begin{table*}[ht]

\caption{Performance comparison between different methods. The Gaia topology is used.
\vspace{-2ex}
}
\begin{center}
\small
\setlength{\tabcolsep}{0.36 em} 
{\renewcommand{\arraystretch}{1.2}
\begin{tabular}{l|c|c|c|ccc|ccc|c|c}
\hline
\multirow{2}{*}{\textbf{Method}} & \multirow{2}{*}{\textbf{Main Focus}} & \multirow{2}{*}{\textbf{Inputs}} & \multirow{2}{*}{\textbf{\begin{tabular}[c]{@{}c@{}}Learning\\ Scenario\end{tabular}}} & \multicolumn{3}{c|}{\textbf{RMSE}} & \multicolumn{3}{c|}{\textbf{MAE}} & \multirow{2}{*}{\textbf{\begin{tabular}[c]{@{}c@{}}\#Params\\ (M)\end{tabular}}} & \multicolumn{1}{c}{\multirow{2}{*}{\textbf{\begin{tabular}[c]{@{}c@{}}Avg. Cycle\\ Time (ms)\end{tabular}}}} \\ \cline{5-10}
 &  &  &  & \multicolumn{1}{c|}{\textit{\textbf{Udacity+}}} & \multicolumn{1}{c|}{\textit{\textbf{Gazebo}}} & \textit{\textbf{Carla}} & \multicolumn{1}{c|}{\textit{\textbf{Udacity+}}} & \multicolumn{1}{c|}{\textit{\textbf{Gazebo}}} & \textit{\textbf{Carla}} &  & \multicolumn{1}{c}{} \\ \hline
\text{MobileNet~\cite{sandler2018mobilenetv2}} & \multirow{5}{*}{Archtecture} & \multirow{2}{*}{\begin{tabular}[c]{@{}c@{}}Realtime\\ Vision\end{tabular}} & \multirow{5}{*}{CLL} & \multicolumn{1}{c|}{0.193} & \multicolumn{1}{c|}{0.083} & 0.286 & \multicolumn{1}{c|}{0.176} & \multicolumn{1}{c|}{0.057} & 0.200 & 2.22 & \_  \\ \cline{1-1} \cline{5-12} 
\text{DroNet~\cite{loquercio2018dronet}} &  &  &  & \multicolumn{1}{c|}{0.183} & \multicolumn{1}{c|}{0.082} & 0.333 & \multicolumn{1}{c|}{0.15} & \multicolumn{1}{c|}{0.053} & 0.218 & 0.31 & \_ \\ \cline{1-1} \cline{3-3} \cline{5-12} 
\text{St-p3~\cite{hu2022st}} &  & \multirow{3}{*}{\begin{tabular}[c]{@{}c@{}}Temporal\end{tabular}} &  & \multicolumn{1}{c|}{0.092} & \multicolumn{1}{c|}{0.071} & 0.132 & \multicolumn{1}{c|}{0.090} & \multicolumn{1}{c|}{0.049} & 0.132 & 1247.87 & \_  \\ \cline{1-1} \cline{5-12} 
\text{ADD~\cite{zhao2021end}} &  &  &  & \multicolumn{1}{c|}{0.097} & \multicolumn{1}{c|}{0.049} & 0.166 & \multicolumn{1}{c|}{0.092} & \multicolumn{1}{c|}{0.042} & 0.121 & 3234.22 & \_ \\ \cline{1-1} \cline{5-12} 
\text{HPO~\cite{abou2019multimodalHPO}} &  &  &  & \multicolumn{1}{c|}{0.088} & \multicolumn{1}{c|}{0.044} & 0.157 & \multicolumn{1}{c|}{0.070} & \multicolumn{1}{c|}{0.044} & 0.105 & 5990.19 & \_  \\ \hline\hline
\text{FedAvg~\cite{mcmahan2017communication}} & \multirow{4}{*}{\begin{tabular}[c]{@{}c@{}}Aggregation/\\ Optimization\end{tabular}} & \multirow{3}{*}{\begin{tabular}[c]{@{}c@{}}Realtime\\ Vision\end{tabular}} & \multirow{8}{*}{SFL} & \multicolumn{1}{c|}{0.212} & \multicolumn{1}{c|}{0.094} & 0.269 & \multicolumn{1}{c|}{0.185} & \multicolumn{1}{c|}{0.064} & 0.222 & 0.31 & \multicolumn{1}{c}{152.4} \\ \cline{1-1} \cline{5-12} 
\text{FedProx~\cite{li2018federated}} &  &  &  & \multicolumn{1}{c|}{0.152} & \multicolumn{1}{c|}{0.077} & 0.226 & \multicolumn{1}{c|}{0.118} & \multicolumn{1}{c|}{0.063} & 0.151 & 0.31 & \multicolumn{1}{c}{111.5} \\ \cline{1-1} \cline{5-12} 
\text{STAR~\cite{sattler2019robust}} &  &  &  & \multicolumn{1}{c|}{0.179} & \multicolumn{1}{c|}{0.062} & 0.208 & \multicolumn{1}{c|}{0.149} & \multicolumn{1}{c|}{0.053} & 0.155 & 0.31 & \multicolumn{1}{c}{299.9} \\ \cline{1-1} \cline{3-3} \cline{5-12} 
\text{FedTSE~\cite{yuan2022fedtse}} &  & \multirow{5}{*}{\begin{tabular}[c]{@{}c@{}}Temporal\end{tabular}} &  & \multicolumn{1}{c|}{0.144} & \multicolumn{1}{c|}{0.063} & 0.079 & \multicolumn{1}{c|}{0.075} & \multicolumn{1}{c|}{0.051} & 0.154 & 89.1 & 1172  \\ \cline{1-2} \cline{5-12} 
\text{TGCN~\cite{yu2018spatio}} & \multirow{2}{*}{Clustering} &  &  & \multicolumn{1}{c|}{0.137} & \multicolumn{1}{c|}{0.069} & 0.193 & \multicolumn{1}{c|}{0.069} & \multicolumn{1}{c|}{0.047} & 0.179 & 78.33 & 224  \\ \cline{1-1} \cline{5-12} 
\text{Fed-STGRU~\cite{kaur2024federated}} &  &  &  & \multicolumn{1}{c|}{0.129} & \multicolumn{1}{c|}{0.059} & 0.151 & \multicolumn{1}{c|}{0.080} & \multicolumn{1}{c|}{0.048} & 0.156 & 91.01 & 370  \\ \cline{1-2} \cline{5-12} 
\text{BFRT~\cite{meese2022bfrt}} & \multirow{2}{*}{Archtecture} &  &  & \multicolumn{1}{c|}{0.113} & \multicolumn{1}{c|}{0.054} & 0.111 & \multicolumn{1}{c|}{0.081} & \multicolumn{1}{c|}{0.043} & 0.133 & 427.26 & 1256 \\ \cline{1-1} \cline{5-12} 
\text{MFL~\cite{zeng2021multi}} &  &  &  & \multicolumn{1}{c|}{0.108} & \multicolumn{1}{c|}{0.052} & 0.133 & \multicolumn{1}{c|}{0.093} & \multicolumn{1}{c|}{0.043} & 0.138 & 173.87 & 781   \\ \hline \hline
\text{CDL~\cite{do2024reducingCDL}} & Optimization & \multirow{4}{*}{\begin{tabular}[c]{@{}c@{}}Realtime\\ Vision\end{tabular}} & \multirow{6}{*}{DFL} & \multicolumn{1}{c|}{{0.141}} & \multicolumn{1}{c|}{{0.062}} & {0.183} & \multicolumn{1}{c|}{{0.083}} & \multicolumn{1}{c|}{{0.052}} & {0.147} & 0.63 & \multicolumn{1}{c}{72.7} \\ \cline{1-2} \cline{5-12} 
\text{MATCHA~\cite{wang2019matcha}} & \multirow{4}{*}{\begin{tabular}[c]{@{}c@{}}Topology\\ Design\end{tabular}} &  &  & \multicolumn{1}{c|}{0.182} & \multicolumn{1}{c|}{0.069} & 0.208 & \multicolumn{1}{c|}{0.148} & \multicolumn{1}{c|}{0.058} & 0.215 & 0.31 & \multicolumn{1}{c}{171.3} \\ \cline{1-1} \cline{5-12} 
\text{MBST~\cite{prim1957shortest,marfoq2020throughput}} &  &  &  & \multicolumn{1}{c|}{0.183} & \multicolumn{1}{c|}{0.072} & 0.214 & \multicolumn{1}{c|}{0.149} & \multicolumn{1}{c|}{0.058} & 0.206 & 0.31 & \multicolumn{1}{c}{82.1} \\ \cline{1-1} \cline{5-12} 
\text{FADNet~\cite{nguyen2022deep}} &  &  &  & \multicolumn{1}{c|}{0.162} & \multicolumn{1}{c|}{0.069} & 0.203 & \multicolumn{1}{c|}{0.134} & \multicolumn{1}{c|}{0.055} & 0.197 & 0.32 & \multicolumn{1}{c}{62.6} \\ \cline{1-1} \cline{3-3} \cline{5-12} 
\text{PriRec~\cite{chen2020practicalPriRec}} &  & \multirow{2}{*}{\begin{tabular}[c]{@{}c@{}} Temporal\end{tabular}} &  & \multicolumn{1}{c|}{0.137} & \multicolumn{1}{c|}{0.066} & 0.196 & \multicolumn{1}{c|}{0.093} & \multicolumn{1}{c|}{0.052} & 0.127 & 325.57 & 272 \\ \cline{1-2} \cline{5-12} 
\text{PEPPER~\cite{belal2022pepper}} & Clustering &  &  & \multicolumn{1}{c|}{0.124} & \multicolumn{1}{c|}{0.055} & 0.115 & \multicolumn{1}{c|}{0.078} & \multicolumn{1}{c|}{0.054} & 0.122 & 89.13 & 438  \\ \hline\hline
\multirow{3}{*}{\text{\textbf{Ours}}} & \multirow{3}{*}{\begin{tabular}[c]{@{}c@{}}Compact\\ Network\end{tabular}} & \multirow{3}{*}{\begin{tabular}[c]{@{}c@{}} Temporal\end{tabular}} & CLL & \multicolumn{1}{c|}{\textbf{0.088}} & \multicolumn{1}{c|}{0.045} & 0.091 & \multicolumn{1}{c|}{0.078} & \multicolumn{1}{c|}{0.039} & 0.114 & 5.01 & \_ \\ \cline{4-12} 
 &  &  & SFL  & \multicolumn{1}{c|}{0.107} & \multicolumn{1}{c|}{0.049} & \textbf{0.072} & \multicolumn{1}{c|}{\textbf{0.069}} & \multicolumn{1}{c|}{\textbf{0.035}} &0.119 & 5.01 & 180  \\ \cline{4-12} 
 &  &  & DFL & \multicolumn{1}{c|}{0.091} & \multicolumn{1}{c|}{\textbf{0.043}} & 0.076 & \multicolumn{1}{c|}{0.076} & \multicolumn{1}{c|}{0.038} & \textbf{0.104} & 5.01 & 121  \\ \hline
\end{tabular}
}
\vspace{-2ex}
\end{center}

\label{tab:sota}
\end{table*}

\section{Experiment}
\subsection{Implementation Details}

\textbf{Dataset.} Udacity+~\cite{udacity2016data}, Gazebo Indoor~\cite{nguyen2022deep}, and Carla Outdoor dataset~\cite{nguyen2022deep} are used as benchmarking datasets, which are similar to setups mentioned in~\cite{nguyen2021federated,nguyen2022deep}. 
To provide temporal information, we further preprocess the training data by chunking videos into multiple consequences. Each consequence includes a current input image, 5 previous frames, and their corresponding 5 past steering angles.


\textbf{Training.} Each local model is trained with a dynamic batch size and an adaptive learning rate, utilizing the RMSprop~\cite{hinton2012rmsprop} optimizer. The training process is executed in decentralized silos, where local updates are periodically transmitted and aggregated following Equation~\ref{eq:aggr_model}. The Early stopping criterion is applied to ensure convergence and 
the simulation setup follows the framework outlined in~\cite{nguyen2022deep}. As in~\cite{marfoq2020throughput}, our experiments explore three federated network topologies: Gaia~\cite{knight2011internetzoo}, the NWS~\cite{awscloud}, and Exodus framework~\cite{knight2011internetzoo}. While we adopt the NWS topology in primary evaluations to reflect real-world cloud-based federated learning scenarios, Gaia and Exodus are analyzed in an ablation study to assess the impact of varying network structures on performance and convergence behavior.

\textbf{Baselines.}
We evaluate our approach across various learning settings, including real-time vision-based and temporal-based methods. For Centralized Local Learning (CLL), we benchmark against MobileNet-V2~\cite{sandler2018mobilenetv2}, Dronet~\cite{loquercio2018dronet}, St-p3~\cite{hu2022st}, ADD~\cite{zhao2021end}, and HPO~\cite{abou2019multimodalHPO}. In the Server-based Federated Learning (SFL) setting, comparisons are made with FedAvg~\cite{mcmahan2017communication}, FedProx~\cite{li2018federated}, STAR~\cite{sattler2019robust}, FedTSE~\cite{yuan2022fedtse}, TGCN~\cite{yu2018spatio}, Fed-STGRU~\cite{kaur2024federated}, BFRT~\cite{meese2022bfrt}, and MFL~\cite{zeng2021multi}. For Decentralized Federated Learning (DFL), we assess performance against MATCHA~\cite{wang2019matcha}, MBST~\cite{marfoq2020throughput}, FADNet~\cite{nguyen2022deep}, PriRec~\cite{chen2020practicalPriRec}, and PEPPER~\cite{belal2022pepper}. We assess model performance using Root Mean Square Error (RMSE) and Mean Absolute Error (MAE). Additionally, we measure computational efficiency by recording the wall-clock time (ms) for training each method on an NVIDIA A100 GPU.

\subsection{Main Results}
Table~\ref{tab:sota} shows a comparison between our approach and state-of-the-art methods, both with and without temporal information. The results demonstrate a clear performance advantage, as our method achieves notably lower RMSE and MAE across all three datasets: Udacity+, Carla, and Gazebo. Besides, we also provide visualization of method comparisons in Fig.~\ref{fig:VisCompare} and in our supplementary video, which emphasizes our approach's effectiveness in optimizing model complexity while maintaining model convergence and real-time performance, making it suitable for deployment in federated autonomous driving scenarios.

\subsection{Ablation Study}

\begin{table}[ht]
\caption{Performance of methods under multi-modality inputs.
}
\begin{center}
\resizebox{\linewidth}{!}{
\setlength{\tabcolsep}{0.15 em} 
{\renewcommand{\arraystretch}{1.2}
\begin{tabular}{c|c|c|c|c|c|c|c|c|c}
\hline
\multirow{3}{*}{\textbf{Method}} & \multirow{3}{*}{\textbf{Type}} & \multicolumn{3}{c|}{\textbf{Inputs}} & \multicolumn{3}{c|}{\textbf{RMSE}} & \multirow{3}{*}{\textbf{\begin{tabular}[c]{@{}c@{}}\#Params\\ (M)\end{tabular}}} & \multicolumn{1}{c}{\multirow{3}{*}{\textbf{\begin{tabular}[c]{@{}c@{}}Avg. \\Inference\\ Time (ms)\end{tabular}}}} \\ \cline{3-8}
 &  & \multicolumn{1}{c|}{\textit{\textbf{\begin{tabular}[c]{@{}c@{}}Current\\ Image\end{tabular}}}} & \multicolumn{1}{c|}{\textit{\textbf{\begin{tabular}[c]{@{}c@{}}Previous\\ Frame\end{tabular}}}} & \textit{\textbf{\begin{tabular}[c]{@{}c@{}}Steering\\ Series\end{tabular}}} & \multicolumn{1}{c|}{\textit{\textbf{Udacity+}}} & \multicolumn{1}{c|}{\textit{\textbf{Gazebo}}} & \textit{\textbf{Carla}} &  & \multicolumn{1}{c}{} \\ \hline

 &  & \checkmark & \checkmark &  & 0.78 & 0.23 & 0.26  & 207.74 & 426\\ 
 & & \checkmark &  & \checkmark & 0.223 & 0.137 & 0.149 & 121.03  & 128  \\ 
\multirow{-3}{*}{\begin{tabular}[c]{@{}c@{}}HPO\\\cite{abou2019multimodalHPO}\end{tabular}} &\multirow{-3}{*}{\begin{tabular}[c]{@{}c@{}}Full-\\Parametrized\\Network\end{tabular}} & \checkmark & \checkmark & \checkmark & {1.127} & {\_} & {0.972} &5,990.19  &\_ \\ \hline  \hline
 & & \checkmark & \checkmark &   & 0.162  & 0.091  & 0.109  & 1.42 & 19 \\ 
 & & \checkmark &  & \checkmark & 0.144 & 0.092 & 0.092 & 0.97 & 17 \\ 
\multirow{-3}{*}{\textbf{Ours}} &\multirow{-3}{*}{\begin{tabular}[c]{@{}c@{}}Compact\\Network\end{tabular}} & \cellcolor[HTML]{EFEFEF}\checkmark & \cellcolor[HTML]{EFEFEF}\checkmark & \cellcolor[HTML]{EFEFEF}\checkmark & \cellcolor[HTML]{EFEFEF}\textbf{0.091} & \cellcolor[HTML]{EFEFEF}\textbf{0.043}  & \cellcolor[HTML]{EFEFEF}\textbf{0.076} & \cellcolor[HTML]{EFEFEF} 5.01  &  \cellcolor[HTML]{EFEFEF}22 \\ \hline
\end{tabular}
}
}
\end{center}
\vspace{-2ex}

\label{tab:spatial_temporal}
\end{table}

\begin{figure}[!ht] 
   \centering
\resizebox{\linewidth}{!}{
\setlength{\tabcolsep}{2pt}
\begin{tabular}{cccc}
\rotatebox[origin=l]{90}{\hspace{0.1cm} \textbf{\small{CDL~\cite{do2024reducingCDL}}}} &
\shortstack{\includegraphics[width=0.33\linewidth]{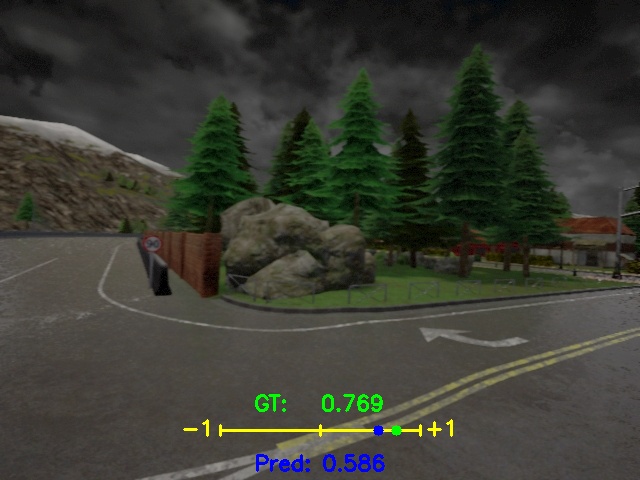}}&
\shortstack{\includegraphics[width=0.33\linewidth]{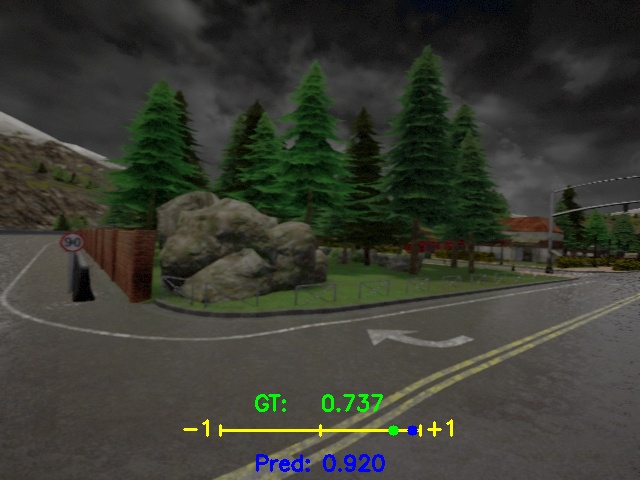}}&
\shortstack{\includegraphics[width=0.33\linewidth]{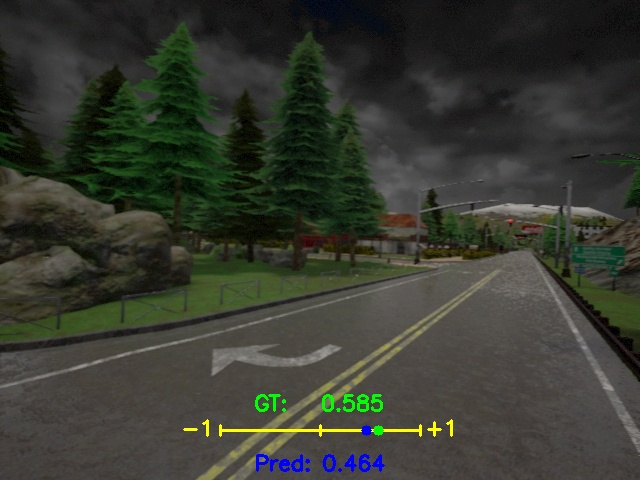}}\\[1pt]
\rotatebox[origin=l]{90}{\hspace{0.35cm} \textbf{\small{HPO~\cite{abou2019multimodalHPO}}}} &
\shortstack{\includegraphics[width=0.33\linewidth]{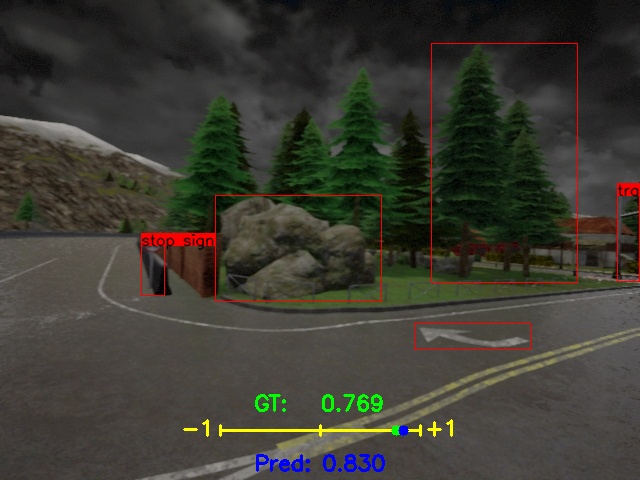}}&
\shortstack{\includegraphics[width=0.33\linewidth]{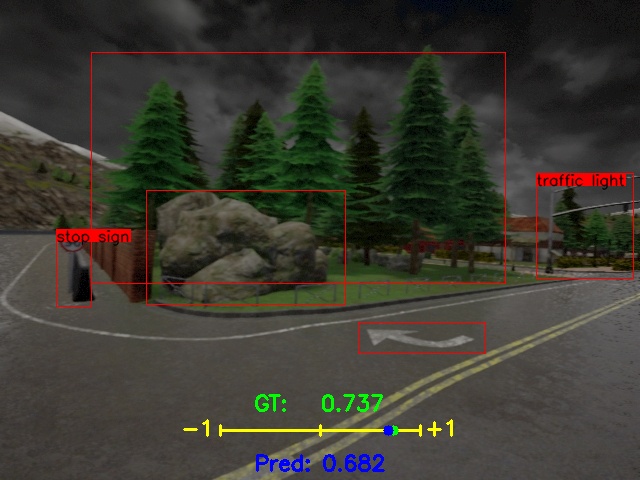}}&
\shortstack{\includegraphics[width=0.33\linewidth]{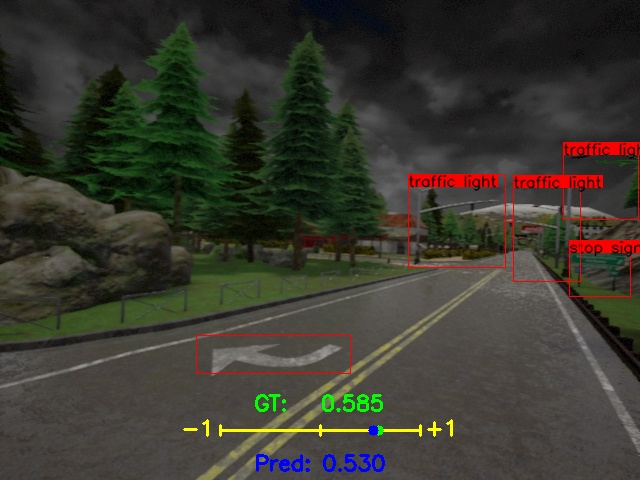}}\\[1pt]
\rotatebox[origin=l]{90}{\hspace{0.5cm} \textbf{\small{Ours}}} &
\shortstack{\includegraphics[width=0.33\linewidth]{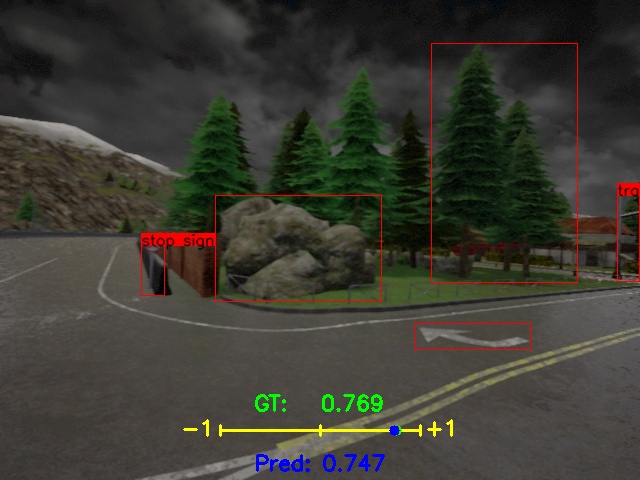}}&
\shortstack{\includegraphics[width=0.33\linewidth]{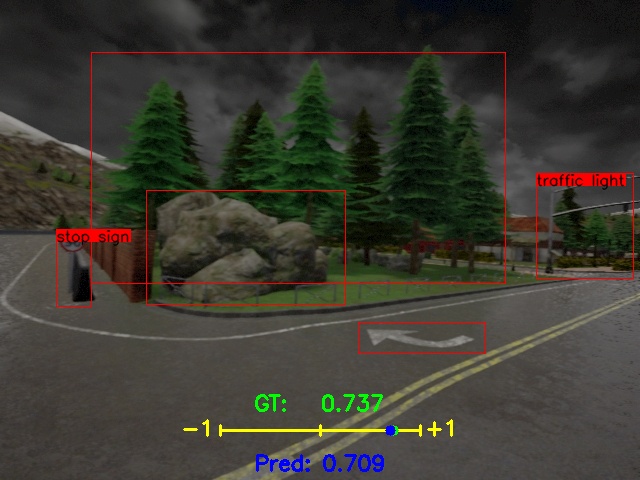}}&
\shortstack{\includegraphics[width=0.33\linewidth]{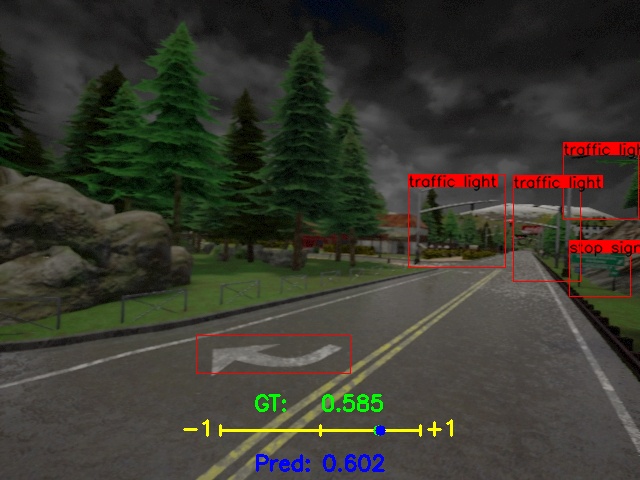}}\\[1pt]

\end{tabular}

}
\vspace{-1ex}
    \caption{Qualitative results between different methods.
    \label{fig:VisCompare}
    \vspace{-3ex}
    } 
\end{figure}

\textbf{Temporal Analysis.} 
Table~\ref{tab:spatial_temporal} shows the performance and the computing trade-off between our compact network and a full-parametrized network HPO~\cite{abou2019multimodalHPO} in handling different modalities for decentralized autonomous driving. The results show that using more temporal information provides higher change to increase model performance. However, they also burst complexity and cause divergence. These results also imply our method's effectiveness in handling complexity while fully leveraging temporal information to maximize performance. Besides, the compact network also ensures convergence and real-time computation.

\textbf{Robustness Analysis.} Training federated algorithms becomes increasingly challenging as the number of vehicle data silos grows. To evaluate the robustness of our approach, we test it alongside baseline methods across different topology sizes. Table~\ref{tab:cross_silo} compares the performance of FADNet~\cite{nguyen2022deep}, CDL~\cite{do2024reducingCDL}, and our method across three network topology infrastructures: Gaia~\cite{knight2011internetzoo} (11 silos), NWS~\cite{awscloud} (22 silos), and Exodus~\cite{knight2011internetzoo} (79 silos). The results indicate that our approach consistently outperforms the baselines in all setups, demonstrating its scalability and effectiveness in large-scale vehicle networks. Moreover, the consistent performance of our approach across different environments, as shown in Fig.~\ref{fig:Vis}, further confirms its robustness and adaptability.




\begin{table}[ht]
\caption{Results of our method under different topologies.}
\begin{center}
\resizebox{\linewidth}{!}{
\setlength{\tabcolsep}{0.25 em} 
{\renewcommand{\arraystretch}{1.1}
\begin{tabular}{c|c|ccc}
\hline
\multirow{2}{*}{\textbf{Topology}} & \multirow{2}{*}{\textbf{Architecture}} & \multicolumn{3}{c}{\textbf{Dataset}} \\ \cline{3-5} 
 &  & \multicolumn{1}{c|}{\textit{Udacity+}} & \multicolumn{1}{c|}{\textit{Gazebo}} & \textit{Carla} \\ \hline
\multirow{3}{*}{\begin{tabular}[c]{@{}c@{}}\textbf{Gaia}\\ \textit{(11 silos)}\end{tabular}} & FADNet & \multicolumn{1}{c|}{0.162\improvecolor($\downarrow$0.071)} & \multicolumn{1}{c|}{0.069\improvecolor($\downarrow$0.026)} & 0.203 \improvecolor($\downarrow$0.127)\\   
 & CDL & \multicolumn{1}{c|}{0.141\improvecolor($\downarrow$0.050)} & \multicolumn{1}{c|}{0.062\improvecolor($\downarrow$0.019)} & 0.183\improvecolor($\downarrow$0.107)\\   
 & \cellcolor[HTML]{EFEFEF}\textbf{{Ours}} & \multicolumn{1}{c|}{\cellcolor[HTML]{EFEFEF}\textbf{0.091}} & \multicolumn{1}{c|}{\cellcolor[HTML]{EFEFEF}\textbf{0.043}} & \cellcolor[HTML]{EFEFEF}\textbf{0.076} \\ \hline
\multirow{3}{*}{\begin{tabular}[c]{@{}c@{}}\textbf{NWS}\\ \textit{(22 silos)}\end{tabular}} & FADNet & \multicolumn{1}{c|}{0.165\improvecolor($\downarrow$0.084)} & \multicolumn{1}{c|}{0.07\improvecolor($\downarrow$0.017)} & 0.2\improvecolor($\downarrow$0.082)\\   
 & CDL & \multicolumn{1}{c|}{0.138\improvecolor($\downarrow$0.057)} & \multicolumn{1}{c|}{0.058\improvecolor($\downarrow$0.005)} & 0.182\improvecolor($\downarrow$0.064)\\   
 & \cellcolor[HTML]{EFEFEF}\textbf{{Ours}} & \multicolumn{1}{c|}{\cellcolor[HTML]{EFEFEF}\textbf{0.081}} & \multicolumn{1}{c|}{\cellcolor[HTML]{EFEFEF}\textbf{0.053}} & 	\multicolumn{1}{c}{\cellcolor[HTML]{EFEFEF}\textbf{0.118}} \\ \hline
\multirow{3}{*}{\begin{tabular}[c]{@{}c@{}}\textbf{Exodus}\\ \textit{(79 silos)}\end{tabular}} & FADNet & \multicolumn{1}{c|}{0.179\improvecolor($\downarrow$0.087)} & \multicolumn{1}{c|}{0.081\improvecolor($\downarrow$0.026)} & 0.238\improvecolor($\downarrow$0.117)\\   
 & CDL & \multicolumn{1}{c|}{0.138\improvecolor($\downarrow$0.046)} & \multicolumn{1}{c|}{0.061\improvecolor($\downarrow$0.006)} & 0.176\improvecolor($\downarrow$0.055)\\   
 & \cellcolor[HTML]{EFEFEF}\textbf{{Ours}} & \multicolumn{1}{c|}{\cellcolor[HTML]{EFEFEF}\textbf{0.092}} & \multicolumn{1}{c|}{\cellcolor[HTML]{EFEFEF}\textbf{0.055}} & \multicolumn{1}{c}{\cellcolor[HTML]{EFEFEF}\textbf{0.121}} \\ \hline
\end{tabular}
}
}
\end{center}
\label{tab:cross_silo}
\vspace{-2ex}
\end{table}

\begin{figure}[h] 
   \centering
\resizebox{1\linewidth}{!}{
\setlength{\tabcolsep}{2pt}
\begin{tabular}{cccccc}
\rotatebox[origin=l]{90}{\hspace{0.3cm} \textbf{Udacity+}} &
\shortstack{\includegraphics[width=0.33\linewidth]{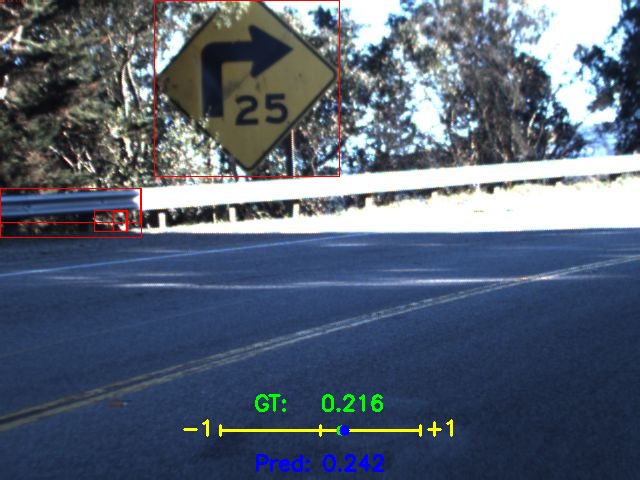}}&
\shortstack{\includegraphics[width=0.33\linewidth]{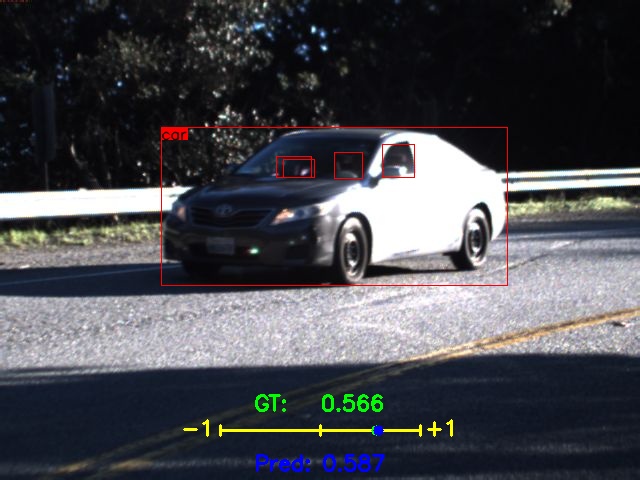}}&
\shortstack{\includegraphics[width=0.33\linewidth]{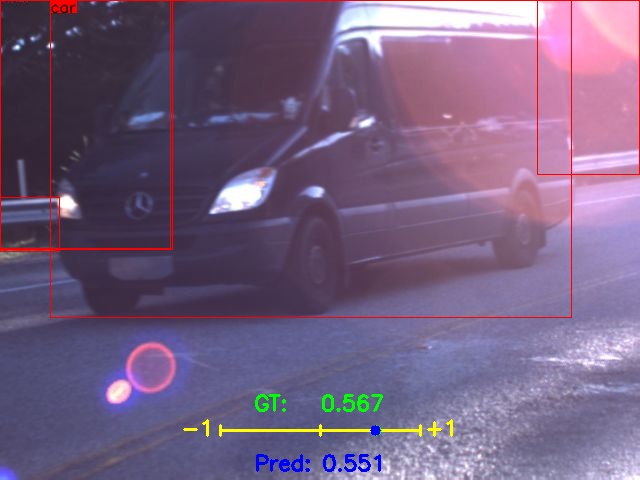}}\\[1pt]
\rotatebox[origin=l]{90}{\hspace{0.2cm} \textbf{Gazebo}} &
\shortstack{\includegraphics[width=0.33\linewidth]{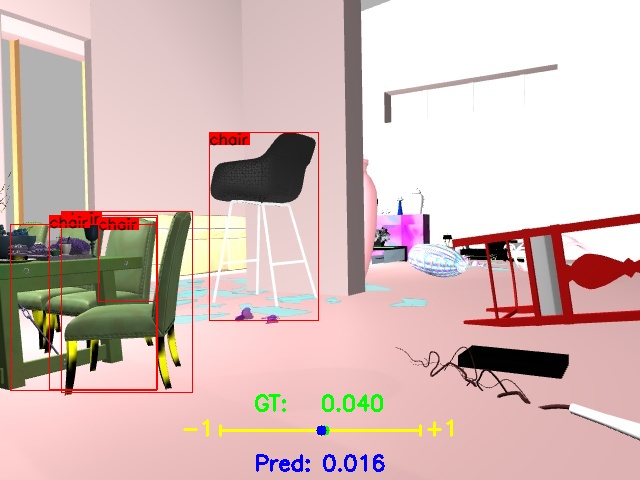}}&
\shortstack{\includegraphics[width=0.33\linewidth]{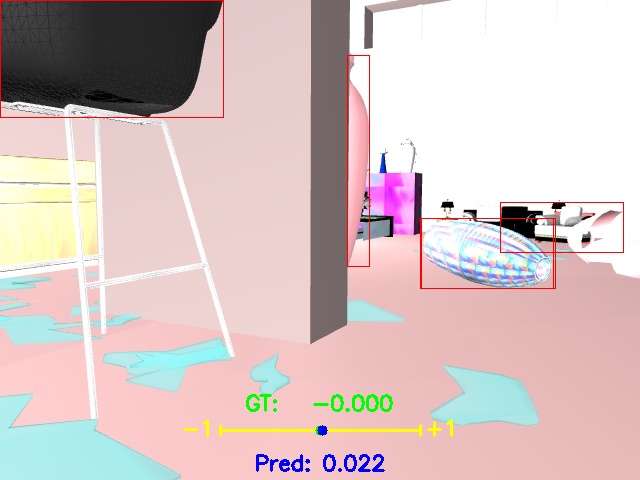}}&
\shortstack{\includegraphics[width=0.33\linewidth]{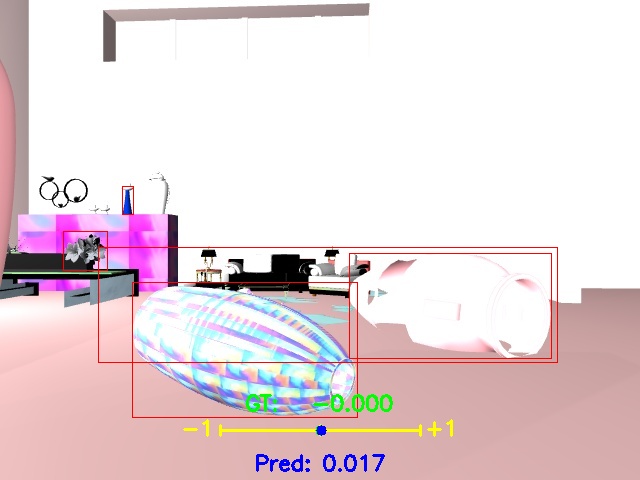}}\\[1pt]
\rotatebox[origin=l]{90}{\hspace{0.2cm} \textbf{Carla}} &
\shortstack{\includegraphics[width=0.33\linewidth]{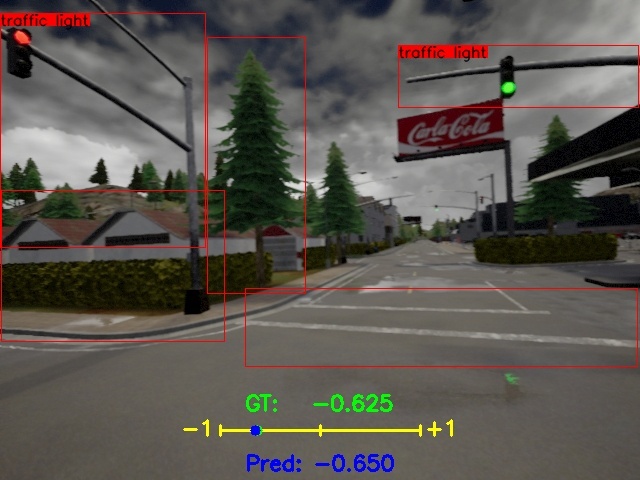}}&
\shortstack{\includegraphics[width=0.33\linewidth]{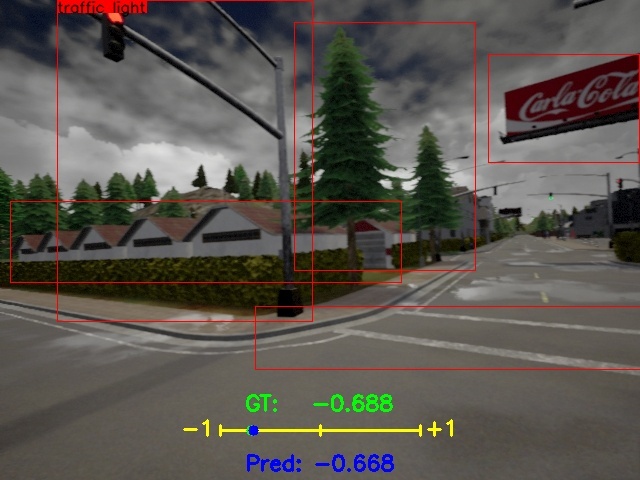}}&
\shortstack{\includegraphics[width=0.33\linewidth]{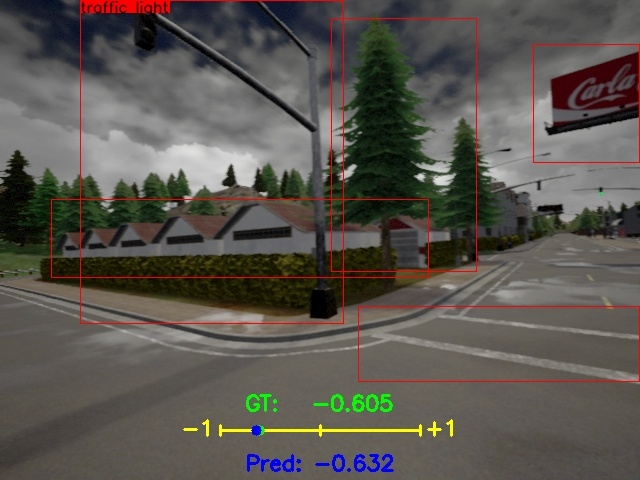}}\\[1pt]

\end{tabular}
}
\vspace{-1ex}
    \caption{Visualization results of our method on different environments. 
    \label{fig:Vis}
    \vspace{-3ex}
    } 
\end{figure}

\textbf{Decomposition Effectiveness.} 
We further compute the decomposition rate of our Lightweight Temporal Transformer Decomposition. For a full interaction between the multimodal inputs with the original vision transformer network, we would need to learn $5.9$ billion parameters, which is infeasible in practice in the federated learning setup. By using the proposed decomposition method with the provided settings, i.e., the number of slicing $\mathcal{R}=32$ and the dimension of the joint representation $d_z=1024$, the number of parameters that need to be learned is only around $5$ million. In other words, we achieve a decomposition rate of approximately $1179$ times.

\subsection{Robotic Demonstration}
We deploy the aggregated trained model on an autonomous mobile platform for real-world validation. The training process utilized the Udacity+ dataset using Gaia topology. The mobile robot is equipped with a 12-core ARM Cortex-A78AE 64-bit CPU and an NVIDIA Orin NX GPU, providing sufficient computational resources for edge-based inference. With an optimized inference time of \textit{18 ms}, our approach enables low-latency, real-time steering angle predictions, crucial for responsive autonomous navigation (Fig.~\ref{fig:VisReal}). Additional real-world demonstrations can be found in our supplementary video.

\begin{figure}[h] 
\vspace{1ex}
   \centering
\resizebox{\linewidth}{!}{
\setlength{\tabcolsep}{2pt}
\begin{tabular}{cc}
\shortstack{\includegraphics[width=0.5\linewidth]{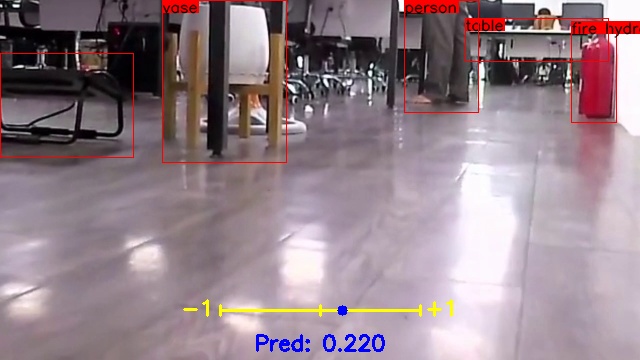}}&
\shortstack{\includegraphics[width=0.5\linewidth]{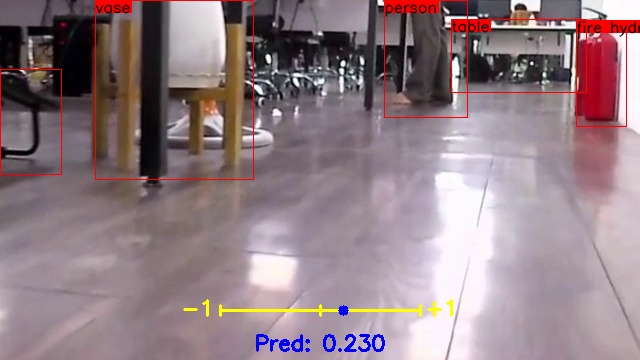}}
\\[0pt]
\shortstack{\includegraphics[width=0.5\linewidth]{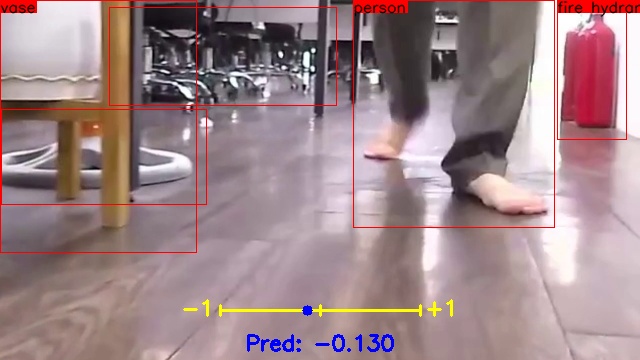}}&
\shortstack{\includegraphics[width=0.5\linewidth]{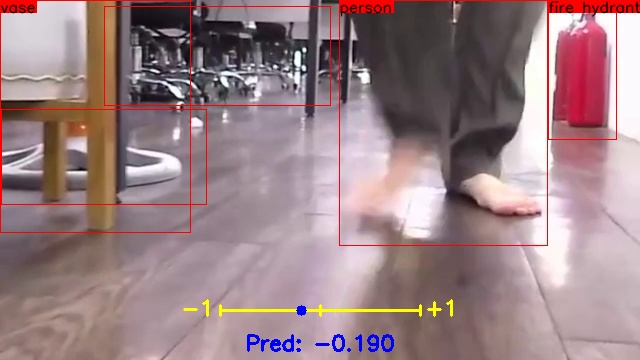}}\\[1pt]

\end{tabular}
}
    \caption{Visualization results in real robot experiments. Our proposed model is lightweight and can be integrated into robot edge devices. 
    \label{fig:VisReal}
    \vspace{-3ex}
    } 

\end{figure}
\subsection{Limitation and Discussion} 
While our proposed lightweight temporal transformer decomposition demonstrates significant improvements in reducing parameter complexity and enhancing computational efficiency, certain limitations remain. 
Specifically, because our method approximates a large transformer into a smaller one with fewer parameters, as the temporal input length increases (e.g., past frames increasing from 5 to 30), the model may be insufficient to learn useful information from the input, potentially leading to degraded performance.
Additionally, the choice of the slicing parameter $\mathcal{R}$ and the embedding dimension $d_z$ impacts the trade-off between accuracy and efficiency. Moreover, the reliance on the rank-1 tensor approximation can lead to a loss of expressiveness, preventing architecture-based solutions from effectively addressing non-independent and identically distributed issues. While our experiments show promising results, addressing these limitations in future work could involve exploring adaptive mechanisms to dynamically adjust the tensor decomposition parameters or integrating regularization techniques to mitigate approximation errors and enhance model stability.

\section{Conclusion}
We propose temporal transformer decomposition, a new method designed to efficiently learn image frames and temporal steering series in a federated autonomous driving context. By leveraging unitary attention decoupling and tensor factorization, our approach decomposes learnable attention maps into small-sized learnable matrices, maintaining an efficient model suitable for real-time predictions while preserving critical temporal information to enhance autonomous driving performance. Extensive evaluations conducted across three datasets demonstrate the effectiveness and feasibility of our approach, validating its potential for practical deployment in autonomous driving systems. In the future, we intend to validate our approach with a broader range of data sources and deploy trained models in more real-world scenarios using autonomous vehicles on public roads.

\bibliographystyle{class/IEEEtran}
\bibliography{class/reference}

\end{document}